\documentclass{article}

\usepackage[preprint]{jmlr2e}

\usepackage{natbib}
\usepackage{amsmath}
\usepackage{amssymb}
\usepackage{cleveref}
\usepackage{amsthm}
\usepackage{subcaption}
\usepackage[T1]{fontenc}
\usepackage[parfill]{parskip}

\theoremstyle{plain}
\newtheorem{theorem}{Theorem}[section]

\newtheorem{lemma}[theorem]{Lemma}
\newtheorem{corollary}[theorem]{Corollary}
\theoremstyle{definition}
\newtheorem{definition}[theorem]{Definition}

\theoremstyle{remark}
\newtheorem{remark}[theorem]{Remark}
\newtheorem*{theorem*}{\textbf{Theorem}}
\newtheorem*{lemma*}{\textbf{Lemma}}
\newtheorem*{corollary*}{\textbf{Corollary}}
\newtheorem*{definition*}{\textbf{Definition}}
\usepackage{tikz}

\title{When Do Off-Policy and On-Policy Policy Gradient Methods Align?}

\author{%
Davide Mambelli \thanks{Correspondence: \texttt{d.mambelli@tudelft.nl}} \hspace{.2mm}\affnum{ 1,2} \hspace{1mm}
Stephan Bongers  \hspace{.2mm}\affnum{1,2}\\
Onno Zoeter \hspace{.2mm}\affnum{2} \hspace{1mm}
Matthijs T.J. Spaan \hspace{.2mm}\affnum{1} \hspace{1mm}
Frans A. Oliehoek \hspace{.2mm}\affnum{1} \\\\
\affiliation{1}{Delft University of Technology, The Netherlands}\\ 
\affiliation{2}{Booking.com, The Netherlands}\\
}

\begin{document}

\maketitle

\begin{abstract}
  Policy gradient methods are widely adopted reinforcement learning algorithms for tasks with continuous action spaces. These methods succeeded in many application domains, however, because of their notorious sample inefficiency their use remains limited to problems where fast and accurate simulations are available. A common way to improve sample efficiency is to modify their objective function to be computable from off-policy samples without importance sampling.
  A well-established off-policy objective is the excursion objective. This work studies the difference between the excursion objective and the traditional on-policy objective, which we refer to as the on-off gap. We provide the first theoretical analysis showing conditions to reduce the on-off gap while establishing empirical evidence of shortfalls arising when these conditions are not met.
\end{abstract}

\section{Introduction}
Policy gradient methods represent a large and popular class of reinforcement learning algorithms \citep{schulman2015, schulman2017} to tackle the set of applications with continuous action spaces, where methods such as Q-learning cannot be directly applied. Robotics \citep{peters2006}, emergent tool usage \citep{baker2020}, and games \citep{openai2019dota} are notable examples of successful applications for these methods. However, policy gradient methods are also notoriously sample-inefficient as they require collecting a large amount of on-policy experience after each parameter update \citep{gu2017}. This is the reason why they are most widely employed in applications where many instances of fast simulators are available.

To go beyond these simulation-based applications and towards real-world deployment, it is necessary to exploit the collected experience efficiently. One solution is to improve sample efficiency by leveraging all available experience. This translates into using off-policy samples, which, for instance, might be gathered by dedicated safe-to-deploy data collection policies, which we will refer to as behavioral policies. In the most general setting, off-policy learning aims at learning under arbitrary behavioral policies. A notable use case consists in learning from a mixture of on- and off-policy samples for example by using experience replay mechanisms \citep{lin1992}, hence avoiding discarding the experience collected when a policy gets updated.

Off-policy policy gradient (OPPG) methods typically rely on a modified version of the on-policy objective function \citep{degris2012}, known as the excursion objective \citep{ghiassian2018, zhang2019}, that can be estimated from off-policy samples avoiding the use of importance sampling techniques, which are known to lead to high variance \citep{liu2018}. However, in general settings, there is no clear understanding of the mismatch between the two objectives. 

The on-policy objective estimates returns achieved by a policy when deployed in the environment, while the off-policy objective, in general, does not describe actual returns achievable in the environment. For this reason, optimizing for the excursion objective does not provide any guarantees of obtaining good performances at deployment time.
\begin{figure}[t]
    \centering
    \begin{tikzpicture}[scale=.7]
        \newcommand\XOFF{-1.5}
        \newcommand\XON{1}
        \draw[samples=50, blue] plot[domain=-1.5:3] ({\x+\XOFF},{2*exp(-1.5*\x*\x)});
        \draw[samples=50, red] plot[domain=-2:2] ({\x+\XON},{1.6*exp(-2*\x*\x)});
        \draw[->] (-3,0) -- node[at end,left,inner sep=2pt] {$J$} (-3,2.7);
        \draw[->] (-3,0) -- node[at end,below,inner sep=3pt] {$\pi$} (3,0);
        \node[blue, above] at (\XOFF -0.5,2) {$J_\text{off}$};
        \node[red, above] at (\XON+0.5,1.6) {$J_\text{on}$};
        \draw[red, dashed] (\XON,0) -- (\XON,1.6);
        \node[below,red] at (\XON,0) {$\pi^*_\text{on}$};
        \node[below,blue] at (\XOFF,0) {$\pi^*_{\text{off}}$};
        \draw[-, dashed, blue] (\XOFF,0) -- (\XOFF,2);
        \node[below,black] at (-0.4,3.7) {$\gamma = 0.9$};
        \draw[<->] (\XOFF,1) -- node[above,inner sep=2pt] {$\Delta$} (\XON,1);
    \end{tikzpicture}
    \begin{tikzpicture}[scale=.7]
        \newcommand\XOFF{-1.3}
        \newcommand\XON{0.5}
        \draw[samples=100, blue] plot[domain=-1.7:3] ({\x+\XOFF},{2*exp(-1.5*\x*\x)});
        \draw[samples=100, red] plot[domain=-2:2.5] ({\x+\XON},{1.6*exp(-2*\x*\x)});
        \draw[->] (-3,0) -- node[at end,left,inner sep=2pt] {$J$} (-3,2.7);
        \draw[->] (-3,0) -- node[at end,below,inner sep=3pt] {$\pi$} (3,0);
        \node[blue, above] at (\XOFF -0.5,2) {$J_\text{off}$};
        \node[red, above] at (\XON+0.5,1.6) {$J_\text{on}$};
        \draw[red, dashed] (\XON,0) -- (\XON,1.6);
        \node[below,red] at (\XON,0) {$\pi^*_\text{on}$};
        \node[below,blue] at (\XOFF,0) {$\pi^*_{\text{off}}$};
        \draw[-, dashed, blue] (\XOFF,0) -- (\XOFF,2);
        \node[below,black] at (-0.5,3.7) {$\gamma = 0.99$};
        \draw[<->] (\XOFF,1) -- node[above,inner sep=2pt] {$\Delta$} (\XON,1);
    \end{tikzpicture}
    \caption{Stylized visualization On-Off Gap vs.\ $\gamma$.}
    \label{fig:overview}
\end{figure}
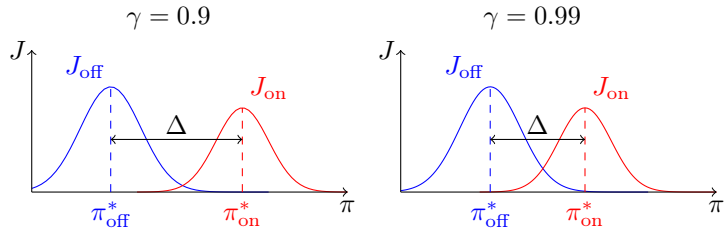

The main contribution of this work is to establish a connection between on- and off-policy gradient methods, by looking at the relation between the objectives they optimize for. In particular, we illustrate the impact of the discount factor on the gap between these two objectives, as illustrated in Figure~\ref{fig:overview}. Specifically, when the discount factor is selected close enough to 1 and the Markov chain is both irreducible and aperiodic we establish that the two objectives and their gradient coincide. We also establish the first upper bound on the norm distance between the on- and off-policy gradients. 
In the end, we empirically verify the impact of the on-off gap in offline policy selection tasks for realistic environments. Our analysis provides theoretical foundations for the use of the excursion objective, improving the understanding of how far what we can compute is from what we want to optimize in OPPG methods.

\section{Background}
In this work, we consider a Markov Decision Process (MDP) $(\mathcal{S},\mathcal{A},T,r,\mu)$ where $\mathcal{S}$ is a finite set of states, $\mathcal{A}$ is the action space, the one-step transition probability is denoted by $T: \mathcal{S}\times\mathcal{A}\rightarrow\Delta(\mathcal{S})$, where $\Delta(\mathcal{S})$ is the space of probability distributions over $\mathcal{S}$, $r: \mathcal{S}\times\mathcal{A}\rightarrow[0,1]$ describes the reward function, and $\mu\in\Delta(\mathcal{S})$ is the initial state distribution \citep{putterman1994}. The discount factor~$\gamma$ is not mentioned in the MDP formulation as we look into different optimality criteria and hence we follow the definition in \citet{putterman1994} by separating the sequential decision process from the optimality criterion.

\subsection{Discounted reward}
Given a discount factor $\gamma\in[0,1)$, for a policy $\pi:\mathcal{S}\rightarrow\Delta(\mathcal{A})$, the value function $V_\pi: \mathcal{S}\rightarrow \mathbb{R}$ is defined as:
\begin{align}
    \label{eq:value_discounted_def}
    V_\pi(s) := \mathbb{E}_{\,\pi, T}\left[\sum_{t=0}^{\infty}\gamma^t r(s_t,a_t) \;\Big|\; s_0=s\right]\,.
\end{align}
Likewise, the action-value function $Q_\pi:\mathcal{S}\times\mathcal{A}\rightarrow\mathbb{R}$ is defined as:
\begin{align}
    Q_\pi(s,a) := \mathbb{E}_{\,\pi, T}\left[\sum_{t=0}^{\infty}\gamma^t r(s_t,a_t) \;\Big|\; s_0=s,a_0=a\right]\,.
\end{align}
Since $r$ is bounded between 0 and 1 both the value and the action value are non-negative and upper bounded by $1/(1-\gamma)$. A core quantity for our work is the state visitation distribution $d_\pi(s)$, which, for discounted rewards, is written as:
\begin{align}
    \label{eq:d_discounted_def}
    d_\pi(s) := (1-\gamma)\sum_{t=0}^{\infty} \gamma^t Pr(s_t = s | s_0\sim\mu,\pi, T)\,.
\end{align}
Similarly to previous definitions, given that the probability $Pr(s_t = s | s_0\sim\mu,\pi, T)$ is trivially upper-bounded by 1 and $\gamma\in[0,1)$ the state visitation always exists and lies in the range $[0,1]$ for all states. Finally, the objective function, describing how a policy performs when selecting actions in the MDP, can be  formulated as:
\begin{align}
    \label{eq:rho_discounted_def}
    \rho(\pi) := \mathbb{E}_{s \sim \mu}\left[ V_\pi(s)\right]\,.
\end{align}
The objective function describes the performances of a policy $\pi$ when it is deployed in the MDP, where the starting state is sampled from the distribution $\mu$.

\subsection{Policy Gradient Methods}
Policy gradient (PG) methods aim to optimize a parametric policy $\pi(a\mid s,\theta)$ to maximize the previously introduced objective~$\rho(\pi)$. Specifically, they compute the gradient of the objective with respect to the policy parameters. We refer to $\rho(\pi)$ as the on-policy objective function and to highlight its on-policy nature we rename it as:
\begin{align}
    J_\mu(\pi) := \rho(\pi)\,,
\end{align}
where the subscript represents the starting state distribution, whereas the argument indicates the policy to evaluate. \citet{sutton1999} established a unified expression for the gradient for stochastic policies:
\begin{align}
    \label{eq:grad_without_logtrick}
    \nabla_\theta J_\mu(\pi) = \sum_s d_\pi  (s) \sum_a Q_\pi(s,a)\nabla_\theta\pi(a|s,\theta) \,.
\end{align}
This expression can be turned into an ``easy-to-estimate'' equivalent one by using the log-trick:
\begin{align}
    \nabla_\theta J_\mu(\pi) = \mathbb{E}_{s\sim d_\pi , \; a \sim \pi} \left[ Q_\pi(s,a)\nabla_\theta\log\pi(a|s,\theta) \right] \,.
\end{align}
In the off-policy setting, experience is gathered independently from $\pi$ following a behavioral (or logging) policy $b$. Hence, the state variable in the available experience is distributed such that $s\sim d_b$. For this reason, the previous expectation cannot be directly estimated from behavioral data. To circumvent this issue, it is possible either to use importance sampling or modify the objective to remove the dependency on $d_\pi $. Note, that in general such a dataset can be collected by multiple policies without specific restrictions on them. In other words, in principle, $d_b$ can be any state distribution. However, to match assumptions by \citet{imani2018} we assume that $d_b$ is the state visitation of a single behavioral policy $b$. This instance of off-policy learning is often referred to as offline because $\pi$ is not used for any online interaction. In this setting a modified objective $J_{d_b}(\pi)$, referred to as the \textit{excursion objective}, has been introduced by \citet{degris2012} dropping the dependency on $d_\pi $:
\begin{align}
    J_{d_b}(\pi) := \mathbb{E}_{s \sim d_b}\left[ V_\pi(s)\right]\,.
\end{align}
The difference between the two objectives $J_\mu$ and $J_{d_b}$ lies in the starting state distribution, and hence the weighting for the value functions at different states. \\
A simple gradient expression for the excursion objective has been at first established by \citet{degris2012} for tabular policy representations\footnote{See Errata by \citet{degris2012} for final theorem statement.}. Several methods were built based on this work such as DPG \citep{silver2014,lillicrap2016}, ACER \citep{wang2017}, and TD3 \citep{fujimoto2018}. These methods are based on semi-gradients as they rely only on an approximation of the true gradient for arbitrary policy representations. \citet{imani2018} showed that previously introduced semi-gradients produce sub-optimal policies and established the gradient expression for any policy representation leveraging emphatic weighting \citep{sutton2016}:
\begin{align}
    \label{eq:imani_grad}
    \nabla_\theta J_{d_b} = \sum_s m(s)\sum_a Q_\pi(s,a)\nabla_\theta\pi(a|s,\theta) \,.
\end{align}
Notably, $m = (\mathbb{I}-\gamma P_\pi)^{-1}\;l$, where $P_\pi$ is the square matrix such that each element is given by $P_\pi(s',s):=\sum_a T(s'|s,a)\pi(a|s)$, and $l\in \mathbb{R}^{|\mathcal{S}|}$ has entries $l(s)=d_b(s)i(s)$ such that $i:\mathcal{S}\rightarrow[0,\infty)$ is an arbitrary state-dependent interest function.\\
Importantly, in practice, Equation~\ref{eq:imani_grad}, contrary to previously established semi-gradients, cannot be straightforwardly estimated from samples because it relies on knowing $P_\pi$ and therefore still lacks applications in reinforcement learning problems. However, the excursion objective remains widely used and its exact gradient retains theoretical interest. 

\section{Problem Formulation and Related Works}
Even though the excursion objective is a well-established element in the policy gradient literature, there is a limited understanding of how well it aligns with its on-policy counterpart, which we refer to as the on-off gap. In the following, we define the problem formally and provide an example showcasing when such an issue arises. \\
Tabular policy representations are known to be complete, meaning that there is a policy $\pi^*$ that simultaneously maximizes $V_\pi(s)$ for all states \citep{kakade2002}. Therefore, $\pi^*$ is optimal under any initial state distribution which implies that when defining a generalized objective function as
\begin{align}
    J_{\nu}(\pi)=\mathbb{E}_{s\sim \nu}[V_\pi(s)],
\end{align}
the state distribution $\nu\in\Delta(\mathcal{S})$ used in the objective function does not affect the optimal policy. However, when restricting the policy class, as it happens for function approximation, objectives with different state distributions might lead to different optimal policies. An example is provided in Figure~\ref{fig:impact_function_approx}, where it is shown that optimizing $J_{\nu}(\pi)$ under different initial state distributions can lead to different optimal policies as long as the policy class is not complete in that it contains the policy that is optimal for all starting states. In particular, the excursion objective might not reflect at all the actual performances of a policy $\pi$ as $d_b$ might be arbitrarily different from $\mu$, causing the optimal policies for the two objectives to not coincide.
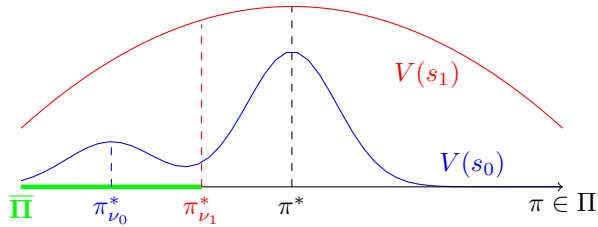
\begin{figure}
    \centering
    \begin{tikzpicture}[scale=1.2]
        \draw[samples=50, blue] plot[domain=-3:3] ({\x},{1.5*exp(-2*\x*\x) + 0.5 * exp(-2*(\x+2)*(\x+2)});
        \draw[samples=50, red] plot[domain=-3:3] ({\x},{0.15*-\x*\x + 2});
        \draw[->] (-3,0) -- node[at end,below,inner sep=2pt] {$\pi\in\Pi$} (3,0);
        \node[blue, above] at (2,0) {$V(s_0)$};
        \node[red, below] at (1.5,1.5) {$V(s_1)$};
        \node[below] at (0,0) {$\pi^*$};
        \draw[-, dashed] (0,0) -- (0,2);
        \draw[-, green, line width=0.6mm] (-1,0) -- node[at end,below,inner sep=2pt] {$\boldsymbol{\overline{\Pi}}$} (-3,0);
        \node[below,blue] at (-2,0) {$\pi^*_{\nu_0}$};
        \draw[-, dashed, blue] (-2,0) -- (-2,0.5);
        \node[below,red] at (-1,0) {$\pi^*_{\nu_1}$};
        \draw[-, dashed, red] (-1,0) -- (-1,1.8);
    \end{tikzpicture}
    \caption{Considering a two-state MDP, we draw $V_\pi(s)$ for each state as a function of the policy $\pi\in\Pi$. Choosing a distribution $\nu\in\Delta(\mathcal{S})$ for the objective $J_{\nu}(\pi)=\mathbb{E}_{s\sim \nu}[V_\pi(s)]$ corresponds to selecting a mixing for the values. We refer to the policy that simultaneously maximizes the value at all states as $\pi^*$. When $\pi^*$ is within the function class, for any distribution $\nu$, we have $\pi^* = \pi^*_{\nu}$, where $\pi^*_{\nu}=\text{argmax}_{\pi\in\Pi}J_{\nu}(\pi)$. Otherwise, different initial state distributions might lead to different optima. For instance, in the example above by selecting two distributions $\nu_0(s_0)=1$ and $\nu_1(s_1)=1$ and restricting the argmax to be over the function class to $\overline{\Pi}\subset\Pi$ (represented by the green line), we have two distinct policies that maximize their respective objective function, $\pi^*_{\nu_0}\neq\pi^*_{\nu_1}$.}
    \label{fig:impact_function_approx}
\end{figure}

The first work looking into conditions for the two objectives to retrieve the same optimal policy is by \citet{laroche2021}. These authors study a generalized version of the policy gradient update $U(\theta, d)$ under arbitrary state distribution $d$:
\begin{align}
    \theta_{t+1}\leftarrow \theta_t + \eta_t U(\theta_t, d_t)
\end{align}
where $U(\theta, d) = \sum_s d(s) \sum_a Q_\pi(s,a)\nabla_\theta\pi(a|s,\theta)$ is a generalization of Equation~\ref{eq:grad_without_logtrick}. In particular, $U(\theta, d)$ represents a parameter update for policy parametrized by $\theta$ given that in the available experience $s\sim d$. For direct and softmax policy parametrization applying exact policy updates induces a sequence of upper bounded and monotonically increasing value functions. This guarantees the convergence of the value function to a fixed point. Moreover, they show that the condition $\forall s \in \mathcal{S}, \; \sum_t\eta_t d_t(s) = \infty$ is necessary and sufficient for the fixed point to be optimal. The impact of this work on OPPG methods is to provide conditions on the behavioral policy such that the updates still converge to the optimal policy. This complements our results as we look for similar guarantees but instead by using conditions on $\pi$, the policy we are trying to improve upon. The only assumption on the behavioral policy we use is $\pi(a|s)>0\rightarrow b(a|s)>0$ to ensure that the off-policy evaluation problem is proper \citep{hallak2017}.

\section{State Visitation and Stationarity}
Our work on connecting on- and off-policy objectives hinges upon establishing a link between state visitation and the stationary properties of the Markov chain that emerges from applying a policy to an MDP. In the following, we discuss how this connection naturally appears when the optimality criterion considers average rewards. However, state visitation is defined instrumentally to the value function, and consequently, different optimality criteria come with different state visitation definitions \citep{sutton1999}. The excursion objective has been introduced only for discounted rewards, where the decaying weights prevent us from directly drawing the connection we are interested in. For this reason, we establish assumptions under which we can restore the link between the state visitation and stationary properties of the Markov chain.

From here onwards, we will refer to the finite-states Markov chain $P_\pi\in \mathbb{R}^{|\mathcal{S}|\times|\mathcal{S}|}$ with entries $P_{\pi}(s',s)=\int_{\mathcal{A}} \pi_\theta(a|s)T(s'|s,a) da$. We present our results for a continuous space since this is a common use case for PG methods.

\subsection{Average Reward}
In the average reward setting \citep[Chapter 8]{putterman1994} the state visitation distribution is defined as 
\begin{align}
    \label{eq:d_average_def}
    d_\pi^{a}(s) := \lim_{t\rightarrow\infty} Pr(s_t=s| s_0\sim\mu,\pi, T)\,,
\end{align}
and, for an MDP with finite states, it can be rewritten as \citep{sutton1999}:
\begin{align}
    \label{eq:d_def_ave}
    d_\pi^{a} = \lim_{t\rightarrow\infty} (P_\pi)^t \mu \,.
\end{align}
Equation~\ref{eq:d_def_ave} is also the definition of the limiting distribution of the Markov chain induced by applying $\pi$ to the MDP. Crucially, the limiting distribution for aperiodic\footnote{A state $s$ such that $P_\pi(s,s)>0$ is said to be aperiodic. A chain is said to be aperiodic if all states are aperiodic.} chains always exists \citep[Section A.4]{putterman1994}. Any distribution $d\in\Delta(\mathcal{S})$ that satisfies the steady state property:
\begin{align}
    \label{eq:stationarity_def}
    d = P_\pi d \, ,
\end{align}
is said to be a stationary distribution for the Markov chain defined by $P_\pi$. This property can be used to compute all the distributions that the system might converge to, by considering the corresponding system of $|\mathcal{S}|$ unknowns with $|\mathcal{S}|$ equations plus the additional equation $\sum_s d(s) = 1$. Following its definition, the limiting distribution is always a stationary distribution. However, in general, a stationary distribution is not always a limiting distribution since different starting states might lead the system to converge to different stationary distributions. If $P_\pi$ is irreducible\footnote{The state-transition graph is strongly connected,  (i.e. from every state it is possible to reach any other state).}, positive recurrent\footnote{If $s_t=u$ define $\tau_u=\min(\hat{t}|s_{t+\hat{t}}=u)$, then a state $u$ is positive recurrent if $E[\tau_u|s_0=u]<\infty$. A chain is said to be positive recurrent if every state is positive recurrent.} and aperiodic then $d = P_\pi d$ has a unique solution. However, if $\mathcal{S}$ is finite and $P_\pi$ is irreducible then all states are positive recurrent \citep{putterman1994}. We summarize this in the following:
\begin{remark}
If $\mathcal{S}$ is finite, $P_\pi$ is irreducible and aperiodic, then the Markov chain has a unique stationary distribution.
\end{remark}
If the stationary distribution is unique, then for any state distributions $\mu$ and $\nu$ it holds that:
\begin{align}
    \lim_{t\rightarrow\infty} (P_\pi)^t \mu = \lim_{t\rightarrow\infty} (P_\pi)^t \nu \,.
\end{align}
For this reason, in the average reward formulation, by assuming irreducibility, we can state that the state visitation distribution, by being a limiting distribution, must also be the unique stationary distribution of the Markov chain.

\subsection{Discounted Reward}
We now analyze whether there exists a link between state visitation distribution and the stationary properties of the underlying Markov chain also when dealing with discounted rewards. Under this optimality criterion, there is, in general, no direct connection between state visitation and limiting distribution. In fact, states affect exponentially less over time the value function, hence, the rewards obtained at states belonging to the limiting distribution, that are visited infinitely often, affect less and less the value. This weighting in the state visitation definition, see Equation~\ref{eq:d_discounted_def}, prevents drawing connections with stationary distributions. We identify conditions such that the state visitation distribution can still be approximated with the limiting distribution of the Markov chain.\\
For finite states, Equation~\ref{eq:d_discounted_def} can be written as:
\begin{align}
    d_\pi^{d} = \sum_{t=0}^{\infty}\gamma^t (P_\pi)^t \mu \, ,
\end{align}
where $\sum_{t=0}^{\infty} \gamma^t (P_\pi)^t$ is guaranteed to converge since $\mathbb{I}-\gamma P_\pi$ is full rank (Theorem 6.1.1 in \citet{putterman1994}). \\
The key insight for our first result is that the relation between $\gamma$ and the convergence speed to the steady-state behavior of the chain affects the overlap between the states visited before the effective horizon $(1-\gamma)^{-1}$ and the start of the steady-state behavior. Thus, the relation between the state visitation and the stationary distributions is controlled by $\gamma$.\\
We start by defining the strong stationary time \citep{LevinPeresWilmer2006} which is crucial in our analysis.
\begin{definition} (Strong stationary time) \\
    In a finite-state Markov chain with transition probabilities described by $P_\pi\in\mathbb{R}^{|\mathcal{S}|\times|\mathcal{S}|}$ and starting distribution $\mu\in\Delta(\mathcal{S})$, we call $T_m$ the strong stationary time if it is the smallest $t$ such that $(P_\pi)^t\mu = (P_\pi)^{t+1}\mu$. For strong stationary time $T_m$ we define $P_\pi^* := (P_\pi)^{T_m}$.
\end{definition}
For finite-state Markov chains, \citet[Proposition 1.10 b]{diaconis1990} proved that a finite strong stationary time always exists. We will extensively use this result throughout this work, which is the main reason for assuming a finite state space $\mathcal{S}$. The strong stationary time acts as a cutoff point between the transient and limiting behaviors. The interplay between $T_m$ and $\gamma$ is what determines if the state visitation distribution is mostly affected by the limiting distribution or by the transitory patterns. In fact, in the following theorem, we show that, for $\gamma$ close to 1, we can approximate the state visitation for the discounted reward setting by the limiting distribution. We highlight that similar results to Lemma~\ref{th:disc_limit_distr} can be found in the literature \citep[Corollary 8.2.5]{putterman1994}, we include our derivation and theorem statement for convenience. 
\begin{lemma}
    \label{th:disc_limit_distr}
    %(Discounted state visitation in the $\gamma$-limit)\\
    For an MDP with finite states and starting distribution $\mu\in\Delta(\mathcal{S})$, it holds that 
    \begin{align}
        \label{eq:th1}
        \lim_{\gamma\rightarrow1}(1-\gamma)\sum_{t=0}^{\infty} \gamma^t (P_\pi)^t \mu = \lim_{t\rightarrow\infty} (P_\pi)^t\mu
    \end{align}
\end{lemma}
The complete proof for this and all the other theoretical results are provided in the supplementary materials. Intuitively, the infinite discounted sum can be split into two sums at $T_m$. The weight of timesteps before $T_m$ goes to zero as $\gamma\rightarrow 1$. Whereas the states for timesteps after $T_m$ are all sampled from the stationary distribution, and therefore since any mixing of the same distribution is the distribution itself, the state visitation is the same as the stationary distribution of the Markov chain. In other words, by the definition of limit, Lemma~\ref{th:disc_limit_distr} let us state that for any $\epsilon>0$ there exists a $\delta>0$ such that if $|\gamma - 1|<\delta$ then:
\begin{align}
    \label{eq:th_1_with_lim_def}
    \left\lVert (1-\gamma)\sum_{t=0}^{\infty} \gamma^t (P_\pi)^t \mu - P_\pi^* \mu \right\rVert < \epsilon \,.
\end{align}
This implies that the closer the discount factor $\gamma$ is to 1, and the shorter the strong stationary time $T_m$ is the more accurate the approximation is going to be. From an intuitive perspective, we are interested in settings where $\gamma\approx1$ and with a short strong stationary time. Hence, almost all states visited before the effective horizon come from the limiting distribution of the chain. Importantly, $\epsilon$ in Equation~\ref{eq:th_1_with_lim_def} depends on $T_m$. The faster the Markov chain reaches the stationary distribution, the smaller $\gamma$ can be 
while achieving the same $\epsilon$-closeness in Equation~\ref{eq:th_1_with_lim_def}.\\ 
The LHS of Equation~\ref{eq:th1} corresponds to Equation~\ref{eq:d_def_ave}, consequently Lemma~\ref{th:disc_limit_distr} connects the definitions of state visitation for the average and discounted reward settings. It follows that, as we show in the following result, Lemma~\ref{th:disc_limit_distr} allows us to remove the dependence on the initial state distribution when $P_\pi$ is irreducible. 
\begin{corollary}
    \label{lemma:s0_indep}
    For an MDP with finite states, if the Markov chain $P_\pi\in \mathbb{R}^{|\mathcal{S}|\times|\mathcal{S}|}$ is irreducible and aperiodic, given any state distributions $\mu\in\Delta(\mathcal{S})$ and $\nu\in\Delta(\mathcal{S})$, it holds that
    \begin{align}
        \lim_{\gamma\rightarrow1}(1-\gamma)\sum_{t=0}^{\infty} \gamma^t (P_\pi)^t \mu = \lim_{t\rightarrow\infty} (P_\pi)^t \nu \,.
    \end{align}
\end{corollary}
The proof is centered on the use of Lemma~\ref{th:disc_limit_distr} to approximate the LHS with a stationary distribution and then using the uniqueness of the stationary distribution to change the starting state distribution.

The results in this section let us connect the definition of state visitation distribution with the limiting distribution of the Markov chain $P_\pi$. As previously discussed, under the irreducibility and aperiodicity of the Markov chain, there is also an equivalence between the stationary and limiting distribution. To conclude, in this section, we connected the state visitation in the discounted reward formulation to the limiting and stationary distributions.

\section{When Do the Objectives Coincide?}
\label{sec:obj_coincide}
Up to now, the excursion objective has been appealing because of its independence from $d_\pi$, which avoids introducing correction factors coming from importance sampling literature \cite{liu2019}, or solving saddle-point optimization problems \cite{nachum2019}. The main goal of this work is to get further insights into the excursion objective understanding when it aligns with its on-policy counterpart, providing conditions for it to describe the performances of a policy when deployed in the real world. If the two objectives can be interchangeably optimized leading to the same final optimal policy, then this ensures equivalent performances between the on- and off-policy setups.

Our analysis is centered around the effect of the discount factor on the choice of the objective function, which, similarly to the value function, scales proportionally with the effective horizon $(1-\gamma)^{-1}$. For this reason, we work only with normalized objectives $J_\nu(\pi) = (1-\gamma)\mathbb{E}_{s\sim \nu}[V_\pi(s)]$ which are bounded between 0 and 1 for any $\nu\in\Delta(\mathcal{S})$. This allows us to study the impact of the discount factor independently on the scale of the objective function. Finally, we work with normalized emphatic weights which are achieved simply by selecting the interest function equal to $1-\gamma$ for all states.

All the following results build on Lemma~\ref{th:disc_limit_distr} and the interplay between the state visitation and the stationary distribution of the Markov chain. To better understand the on-off gap, we first look into the difference between the two objectives which let us derive the following result.
\begin{theorem}
    \label{th:equiv_value}
    For an MDP with finite states, if the Markov chain $P_\pi\in \mathbb{R}^{|\mathcal{S}|\times|\mathcal{S}|}$ is irreducible and aperiodic then for $\gamma\rightarrow1$ it holds that
    \begin{align}
        J_{d_b}(\pi) = J_\mu(\pi) \,,
    \end{align}
    for $J_{d_b}$ and $J_\mu$ the normalized objectives.
\end{theorem}
Theorem~\ref{th:equiv_value} confirms that the two objectives are equal for $\gamma\rightarrow1$. This result draws a connection between the two objectives, which links the excursion objective with the actual deployment performance of a policy. Hence, it hints that for $\gamma$ close enough to 1, if the chain has a unique stationary distribution, it is sufficient to optimize $J_{d_b}(\pi)$ to improve the performances at deployment time. However, in general, the derivative and the limit do not commute 
\begin{align}
    \nabla_\theta \lim_{\gamma\rightarrow 1} J(\pi_\theta) \neq \lim_{\gamma\rightarrow 1} \nabla_\theta J (\pi_\theta)
\end{align}
meaning that the result in Theorem~\ref{th:equiv_value} about the value of the objectives does not directly translate to the gradients. Ultimately we aim to find conditions under which OPPG methods would retrieve the same optimal policy as their on-policy counterpart when provided with the same initialization and value function estimation. 
For this reason, we provide additional results by studying the difference between the two gradients. 
\begin{theorem}
    \label{th:grad_ub}
    For an MDP with finite states, compact action space $\mathcal{A}\subseteq [a_l, a_u]^k$, and bounded p-norm $\left\lVert\nabla_\theta\pi_\theta(a|s) \right\rVert_p\; \forall (s,a)\in\mathcal{S}\times\mathcal{A}$, we have
    \begin{align}
        \left\lVert \nabla_\theta J_{d_b} (\pi_\theta) - \nabla_\theta J_\mu (\pi_\theta)\right\rVert_p \leq 2C(a_u - a_l)^k|\mathcal{S}|^{3/2} d_{TV}(d_b||\mu) \,,
    \end{align}
    where $C = \max_{(s,a)}\left\lVert\nabla_\theta\pi_\theta(a|s) \right\rVert_p$.
\end{theorem}
Thanks to the dependency on $d_{TV}(d_b||\mu)$, this result aligns with the common interpretation of the excursion objective, where the state visitation of the behavioral policy $d_b$ is considered as the initial state distribution from which an infinite excursion with $\pi$ is performed. Interestingly, for a discount factor close enough to 1, the dependence on the initial state distribution can be dropped by assuming the irreducibility of the Markov chain. To highlight such dependence for the bound, we derive a modified version of Theorem~\ref{th:grad_ub} adding irreducibility and aperiodicity assumptions. 
\begin{theorem}
    \label{th:grad_ub_w_ergodicty}
    For an MDP with finite states, compact action space $\mathcal{A}\subseteq [a_l, a_u]^k$, and bounded p-norm $\left\lVert\nabla_\theta\pi_\theta(a|s) \right\rVert_p\; \forall (s,a)\in\mathcal{S}\times\mathcal{A}$, if the Markov chain $P_\pi\in \mathbb{R}^{|\mathcal{S}|\times|\mathcal{S}|}$ is irreducible and aperiodic then we have
    \begin{align}
        \left\lVert \nabla_\theta J_{d_b} (\pi_\theta) - \nabla_\theta J_\mu (\pi_\theta)\right\rVert_p \leq (1-\gamma)2T_mC(a_u - a_l)^k|\mathcal{S}|^{3/2} d_{TV}(d_b||\mu) \,,
    \end{align}
    where $C = \max_{(s,a)}\left\lVert\nabla_\theta\pi_\theta(a|s) \right\rVert_p$.
\end{theorem}
Irreducibility together with aperiodicity allows us to obtain a bound that is tighter for $(1-\gamma) T_m < 1$, which can be rewritten as a condition on the discount factor by solving for $\gamma$. Hence, Theorem~\ref{th:grad_ub_w_ergodicty} is tighter than Theorem~\ref{th:grad_ub} for $\gamma>(T_m -1)/T_m$, while having the desired dependence on the discount factor explicitly stated. Interestingly, also the significance of the strong stationary time $T_m$ appears, confirming the intuition that it is desirable to have a small $T_m$ and a discount factor close to 1. For $\gamma\rightarrow1$ we show that, similarly to the values, also the gradients coincide.
\begin{corollary}
    \label{cor:grad_equiv}
    For an MDP with finite states,  compact action space $\mathcal{A}\subseteq [a_l, a_u]^k$, and bounded p-norm $\left\lVert\nabla_\theta\pi_\theta(a|s) \right\rVert_p\; \forall (s,a)\in\mathcal{S}\times\mathcal{A}$, if the Markov chain $P_\pi\in \mathbb{R}^{|\mathcal{S}|\times|\mathcal{S}|}$ is irreducible and aperiodic, then for $\gamma\rightarrow1$ it holds that
    \begin{align}
        \nabla_\theta J_{d_b} (\pi_\theta) = \nabla_\theta J_\mu(\pi_\theta)
    \end{align}
\end{corollary}
The proof revolves around taking the limit for $\gamma\rightarrow1$ in the upper bound provided in Theorem~\ref{th:grad_ub_w_ergodicty}.

Corollary~\ref{cor:grad_equiv} suggests that by choosing the appropriate $\gamma$ of the Markov chain, we can make gradients arbitrarily close to each other. The equality holds only in the limit but, similarly to Theorem~\ref{th:equiv_value}, it also provides a receipt to mitigate the gap between the two objectives. In particular, given that for $\gamma\rightarrow1$ then $\nabla_\theta J_{d_b} (\pi_\theta)= \nabla_\theta J_\mu(\pi_\theta)$ for the $(\epsilon,\delta)$-definition of limit we can say that given an $\epsilon>0$ we can always find a $\delta>0$ such that $0<|\gamma-1|<\delta$ implies $|\nabla_\theta J_{d_b} (\pi_\theta)- \nabla_\theta J_\mu(\pi_\theta)|< \epsilon$. By constraining the proximity between the two gradients with a user-defined parameter $\epsilon$, we can establish the existence of a $\gamma$ sufficiently close to 1 that guarantees the desired $\epsilon$-closeness between the gradients. Importantly, the tightness of the bounds does not affect our conclusions since when a loose upper bound of a positive quantity converges to zero then also tighter bounds must converge at least at the same rate. 
\paragraph{Summary} Our first result connects the discounted state visitation to the limiting distribution of the Markov chain for $\gamma$ approaching 1. This instrumental result, together with irreducibility and aperiodicity, ends up unifying the on- and off-policy objectives (and their gradient) when $\gamma$ tends to 1. The importance of these results lies in identifying assumptions under which the two objectives coincide.

\section{Experiments}
In Section~\ref{sec:obj_coincide}, we established a theoretical connection between excursion and on-policy objective. In the following, we first do an empirical validation for our theoretical result in an MDP matching all our assumptions. We then show how our results have a significant impact on realistic environments. We refer to Appendix~\ref{sec:app_exp_setup} for details on the experimental setup.
\subsection{Empirical Validation}
\begin{figure}[t]
\centering
\begin{subfigure}{0.35\textwidth}
    \includegraphics[width=\textwidth]{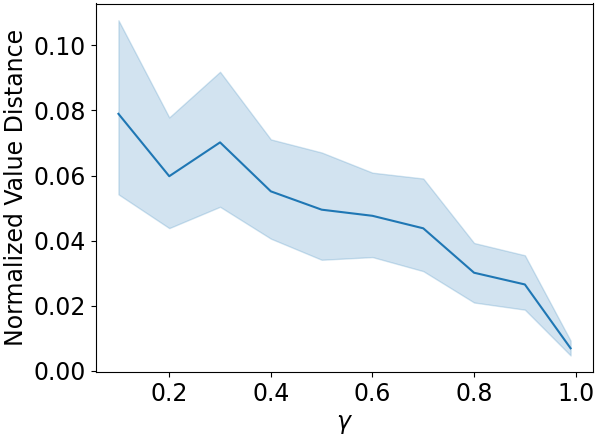}
    \caption{}
    \label{fig:mad}
\end{subfigure}
\begin{subfigure}{0.35\textwidth}
    \includegraphics[width=\textwidth]{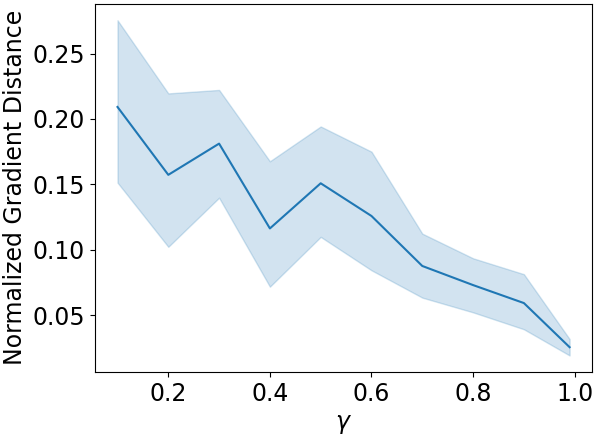}
    \caption{}
    \label{fig:grad_error}
\end{subfigure}
\caption{On-Off gap as mean and 95\% confidence interval: (a) Dependence of value error, computed as $(1-\gamma)|J_\mu(\pi)-J_{d_b}(\pi)|$, on $\gamma$. (b) Dependence of gradient error, computed as $(1-\gamma)\lVert \nabla J_\mu(\pi)-\nabla J_{d_b}(\pi)\lVert $, on $\gamma$.}
\label{fig:empirical_results}
\end{figure}

We start by looking into how the on-off gap behaves as a function of $\gamma$ for a simple 2-state MDP that matches all our assumptions. We select a behavioral policy $b$ and, for each value of $\gamma$, we compute the mismatch between the two objectives and their gradients over 25 randomly sampled policies $\pi$. Figure~\ref{fig:mad} empirically shows that the gap approaches 0 for $\gamma$ close to 1, which reflects what Theorem~\ref{th:equiv_value} theoretically establishes. Figure~\ref{fig:mad} also provides empirical evidence for the trend across the whole range of $\gamma$ going beyond the result in the limit described in Theorem~\ref{th:equiv_value}. Similarly, we look at how the gradient of the two objectives differs across a range of policies for different values of $\gamma$. Figure~\ref{fig:grad_error} highlights a similar trend as for the value of the objectives, as the distance between the gradients decreases to zero as the $\gamma$ approaches 1, which again matches our previous results. We recall that we are dealing with normalized objectives ranging between 0 and 1 to have a meaningful comparison of errors across different values of $\gamma$.

\subsection{Practical Implications}
\label{subsection:practical}

So far, we have theoretically and empirically verified that there exists a $\gamma$-dependent gap between the two objectives. However, such a gap might not be significant for practical applications. For this reason, in the following, we assess the practical implications of our results by studying whether neglecting this gap can be harmful in practical applications.

The off-policy gradient \citep{imani2018} is not well-suited to be directly computed from samples as it requires estimating the transition probability matrix $P_\pi$. Therefore, we select a different use case of the excursion objective beyond PG methods to show the practical implications of our results.  We look into the problem of \textit{offline policy selection} \citep{yang22a}, which refers to the task of selecting the best-performing policy from a finite set of known policies when only data from a behavioral policy is available. In particular, we show that, even for continuous state spaces, only when selecting high enough values of $\gamma$ the off-policy objective allows us to retrieve the same ranking of policies as its on-policy counterpart. To prove this point, we look into the impact of the discount factor on Kendall's $\tau$ correlation coefficient \citep{kendall38} computed between the true ranking obtained using the on-policy objective and the ranking predicted by the excursion objective. Kendall's $\tau$ is high when the generated rankings are similar and 1 when the same. Once an estimate of the value function is obtained through off-policy evaluation methods, the value of the excursion objective can be directly estimated from off-policy data without relying on importance sampling. Thus, it is of high practical interest to understand to what extent the excursion objective can be used for offline policy selection.
\begin{figure*}[t]
\centering
\includegraphics[width=0.96\textwidth]{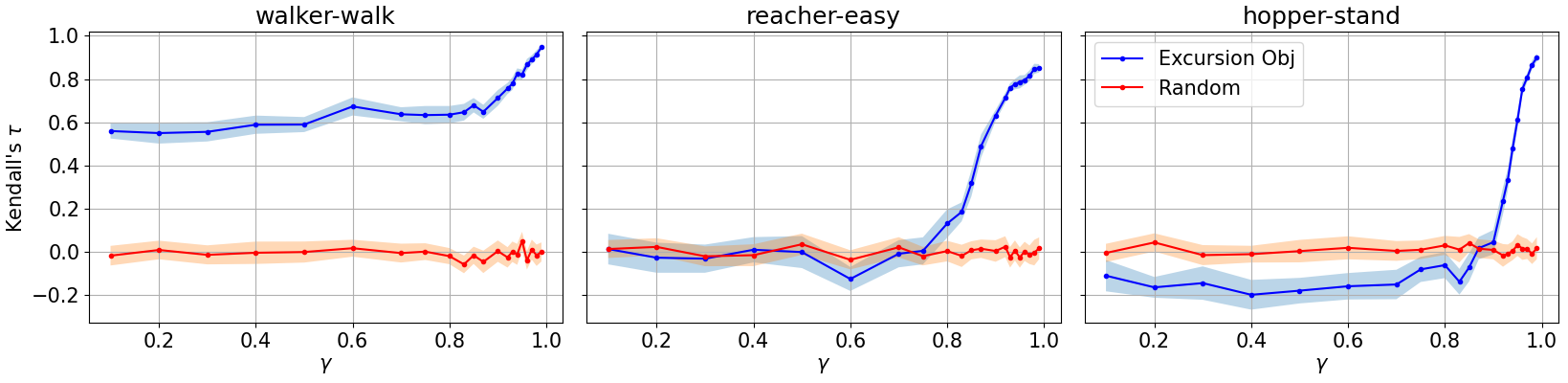}
\caption{Performances in offline policy selection (and 95\% confidence intervals) across 3 Mujoco environments as a function of $\gamma$ when ranking policies with the excursion objective or randomly, with 1 being the same ranking as ground truth.}
\label{fig:mujoco_exp}
\end{figure*}
In our experiments, we use environments from DeepMind Control Suite \citep{tunyasuvunakool2020}. These environments are irreducible and aperiodic when using stochastic policies with support over the whole action space. Irreducibility is given by the possibility of reaching any state from any other state, while aperiodicity is ensured by the non-zero probability of transitioning from one state to itself. Importantly, these environments are characterized by continuous state spaces, which violates one of our main assumptions. However, in Figure~\ref{fig:mujoco_exp} it is clear that $\gamma$ still impacts the ranking performances obtained by the excursion objective as our theoretical results predict. In particular, using a higher discount factor still enables the excursion objective to accurately describe deployment performances by consistently selecting the best-performing policy. For this reason, we conjecture our results hold in the more general setting of continuous state spaces, but we leave a theoretical analysis of this setting to future works. 

In Figure~\ref{fig:mujoco_exp}, we can notice how in all three environments the excursion objective provides a more accurate ranking as $\gamma$ increases. While in Walker the excursion objective allows to obtain good ranking performances throughout the whole range of $\gamma$, both in Reacher and Hopper the two objectives quickly reach a score around zero which is as good as choosing a random ranking. For Hopper, for $\gamma<0.87$ we observe even a negative score, which means that it would be preferable to rank the policies according to the opposite of the excursion objective. Figure~\ref{fig:mujoco_exp} shows how misleading it can be to optimize the excursion objective even at high $\gamma$ regimes, as it can lead to ranking policies worse than random instead of aligning with the on-policy objective. This can lead to selecting highly sub-optimal or even choosing randomly, policies in practice. In Hopper, to reach a ranking score of 0.6, which is equivalent to ranking correctly 80 pairs of policies out of 100, it is necessary to have $\gamma\geq0.95$. For a discussion about the link between the ranking metrics analyzed here and the statistical correlation between the two objectives, we point the reader to Appendix~\ref{subsec:Mujoco} (see discussion Figure~\ref{fig:mujoco_pvalue})

\textbf{Summary:} Even for continuous action spaces, our experiments show that there are environments where selecting $\gamma$ as high as 0.95 can still lead to a significant difference between the two objectives. This highlights the danger of using the excursion objective for $\gamma$ not close enough to 1. 

\section{Discussion}

Our results concern the difference between the two gradients at each policy update, for them to also apply to a policy gradient algorithm it is necessary to ensure that each policy visited throughout the policy optimization path satisfies our assumptions. Importantly, this can be ensured, as mentioned in Subsec.~\ref{subsection:practical}, in many practical settings with simple conditions on the policy class. For this reason, our results apply beyond ergodic MDPs which restrict the chain structure for all policies \citep[section 8.3.1]{putterman1994}. Instead, we require only the existence of some parametrizable policy for each $s_1,s_2\in\mathcal{S}$ that takes from $s_1$ to $s_2$, which represents a subset of the more general class of communicating MDPs.

It is noteworthy that choosing high discount factors is known to lead to slower convergence due to the increase in the effective horizon, see \citet[Section 4.1]{Dewanto2022} for a detailed discussion. From this, we conclude that in OPPG methods $\gamma$ can be used to trade off the alignment between the two objectives and the convergence speed, which is a new important insight following our results.

Policy improvement and value estimation are usually combined in the actor-critic framework. An important remark is that in our analysis we focused on conditions on the policy improvement step but we neglected the off-policy evaluation sub-problem and the additional complexity introduced by estimating the value off-policy \citep{hallak2017}. 

Our results suggest that the discount factor has a significant impact on OPPG methods relying on the excursion objective. The equivalence in the limit $\gamma\rightarrow 1$ shows that the gap between on- and off-policy methods naturally vanishes when optimizing for stationary behaviors as in the average reward formulation. Hence, a direct consequence of our analysis is that the excursion objective is particularly well suited for problems where it is possible to use the average reward formulation \citep{saxena2023}. Similar conclusions about the benefits of using the average reward formulation have been raised also for methods based on approximate Dynamic Programming methods \citet{Dewanto2022}. 

\citet{thomas2014} shows that several natural actor-critic methods, relying on approximate gradients of the on-policy objective, are biased not optimizing the desired objective. Instead, they optimize for the average reward objective. In our work, instead, we study the excursion objective and provide conditions for it to align with the on-policy objective, concluding that this is the case under the uniqueness of the stationary distribution and $\gamma\rightarrow 1$. This coincides with optimizing for stationary behaviors. For these reasons, both works provide pieces of evidence of the practical advantages of optimizing for the average reward objective. 

\section{Conclusions}
This work bridges the gap between two objectives used in policy gradient methods for the on- and off-policy learning contexts. In particular, we identified that the discount factor, strong stationary time, irreducibility, and aperiodicity of the Markov chain play a crucial role when dealing with the excursion objective. We establish that under assumptions on these elements, the two objectives can be made arbitrarily close and in the limit coincide. Therefore, we provide a theoretical justification for the off-policy objective, which so far has been established and used for its empirical advantages \citep{degris2012}. Our main theoretical result highlights that the excursion objective and its gradient can be made arbitrarily close to their on-policy counterparts by selecting a discount factor close to 1. This work importantly connects the objective used to evaluate policies in off-policy reinforcement learning and the actual deployment performance which is typically the main optimization goal of the reinforcement learning paradigm. 

Our experimental results demonstrate that the gap between these objectives has implications in practical settings to the point that choosing policies at random can be preferable over the excursion objective. Our analysis, to the best of our knowledge, is the first work identifying and characterizing this gap, showing its impact in realistic settings and providing conditions for the alignment between these objectives.

\section{Acknowledgements}
This work is supported by Booking.com. We thank Mustafa Mert Çelikok, Jinke He, Elena Congeduti, Zuzanna Osika, Oussama Azizi, Pascal van der Vaart, Bram van den Akker, Audrey Poinsot, Ariyan Bighashdel, and anonymous reviewers for their feedback on the manuscript.

\newpage
\bibliographystyle{apalike}
\bibliography{arxiv}

\onecolumn

\section{Appendix A - Proofs}
We recall a definition from the main paper that is essential for all our results.
\begin{definition*} (Strong stationary time)
    In a finite-state Markov chain with transition probabilities described by $P_\pi\in\mathbb{R}^{|\mathcal{S}|\times|\mathcal{S}|}$ and starting distribution $\mu\in\Delta(\mathcal{S})$, we call $T_m$ the strong stationary time if it is the smallest $t$ such that $(P_\pi)^t\mu = (P_\pi)^{t+1}\mu$. For strong stationary time $T_m$ we define $P_\pi^* := (P_\pi)^{T_m}$.
\end{definition*}
\subsection{Proof Lemma~\ref{th:disc_limit_distr}}
\begin{theorem*}
    (Discounted state visitation in the $\gamma$-limit)\\
    For an MDP with finite states and starting distribution $\mu\in\Delta(\mathcal{S})$, it holds that 
    \begin{align}
        \lim_{\gamma\rightarrow1}(1-\gamma)\sum_{t=0}^{\infty} \gamma^t (P_\pi)^t \mu = \lim_{t\rightarrow\infty} (P_\pi)^t\mu
    \end{align}
\end{theorem*}
\begin{proof}
    We start by dividing the infinite sum at the strong stationary time to isolate the transient part of the Markov chain.
    \begin{align}
        (1-\gamma)\sum_{t=0}^{\infty} \gamma^t (P_\pi)^t \mu &= (1-\gamma)\left[\sum_{t=0}^{T_m-1} \gamma^t (P_\pi)^t \mu + \sum_{t=T_m}^{\infty} \gamma^t (P_\pi)^t \mu \right]  \\
        ((P_\pi)^t\mu =P_\pi^*\mu \;\; \forall t>T_m) \;\;\; &= (1-\gamma)\left[\sum_{t=0}^{T_m-1} \gamma^t (P_\pi)^t \mu + \sum_{t=T_m}^{\infty} \gamma^t P_\pi^* \mu  \right]
    \end{align}
    We can rewrite the second term inside the parentheses as a subtraction of two sums.
    \begin{align}
        (1-\gamma)\sum_{t=0}^{\infty} \gamma^t (P_\pi)^t \mu&= (1-\gamma)\left[\sum_{t=0}^{T_m-1} \gamma^t (P_\pi)^t \mu + 
        \left[\sum_{t=0}^{\infty} \gamma^t - \sum_{t=0}^{T_m-1} \gamma^t\right]P_\pi^* \mu\right]\\
        &= (1-\gamma)\left[\sum_{t=0}^{T_m-1} \gamma^t (P_\pi)^t \mu + 
        \left[\frac{1}{1-\gamma} - \frac{1-\gamma^{T_m}}{1-\gamma}\right]P_\pi^* \mu\right]\\
        &= (1-\gamma)\left[\sum_{t=0}^{T_m-1} \gamma^t (P_\pi)^t \mu + 
        \frac{\gamma^{T_m}}{1-\gamma}P_\pi^* \mu\right]\\
        &=c(\gamma) + \gamma^{T_m} P_\pi^* \mu
    \end{align}
    where we defined $c(\gamma)=(1-\gamma)\sum_{t=0}^{T_m-1} \gamma^t (P_\pi)^t \mu$. Observe that $\lim_{\gamma\to 1} c(\gamma) = 0$ and hence:
    \begin{align}
    \lim_{\gamma\rightarrow1}(1-\gamma)\sum_{t=0}^{\infty} [ \gamma^t (P_\pi)^t \mu ]&=\lim_{\gamma\to 1} \left[c(\gamma) + \gamma^{T_m} P_\pi^* \mu \right] \\
    &= P_\pi^*\mu \\
    &= \lim_{t\rightarrow\infty} (P_\pi)^t\mu \,.
    \end{align}
\end{proof}
\subsection{Proof Corollary~\ref{lemma:s0_indep}}
\begin{corollary*}
    For an MDP with finite states, if the Markov chain $P_\pi\in \mathbb{R}^{|\mathcal{S}|\times|\mathcal{S}|}$ with entries $P_{\pi}(s',s)=\int_{\mathcal{A}} \pi_\theta(a|s)T(s'|s,a) da$ is irreducible and aperiodic, given any state distributions $\mu\in\Delta(\mathcal{S})$ and $\nu\in\Delta(\mathcal{S})$, it holds that
    \begin{align}
        \lim_{\gamma\rightarrow1}(1-\gamma)\sum_{t=0}^{\infty} \gamma^t (P_\pi)^t \mu = \lim_{t\rightarrow\infty} (P_\pi)^t \nu \,.
    \end{align}
\end{corollary*}
\begin{proof}
From Lemma~\ref{th:disc_limit_distr} we have:
    \begin{align}
        \lim_{\gamma\rightarrow1}(1-\gamma)\sum_{t=0}^{\infty} \gamma^t (P_\pi)^t \mu &= \lim_{t\rightarrow\infty} (P_\pi)^t\mu
    \end{align}
    By definition of strong stationary time we have that for any $t>T_m$, then $(P_\pi)^t=P_\pi^*$.
    \begin{align}
        \lim_{\gamma\rightarrow1}(1-\gamma)\sum_{t=0}^{\infty} \gamma^t (P_\pi)^t \mu &= P_\pi^*\mu\\
        (\text{Irreducibility + Aperiodicity})\;\;\;&= P_\pi^*\nu
    \end{align}
    where in the last equality we used the fact that $P_\pi^*$ has identical columns\footnote{Note that here we are using the transpose of the notation used in Appendix A \cite{putterman1994}.} when both irreducibility and aperiodicity assumptions are met.\\
\end{proof}
\subsection{Proof Theorem~\ref{th:equiv_value}}
\begin{theorem*}
    (Objectives equivalence in the $\gamma$-limit)\\
    For an MDP with finite states, if the Markov chain $P_\pi\in \mathbb{R}^{|\mathcal{S}|\times|\mathcal{S}|}$ with entries $P_{\pi}(s',s)=\int_{\mathcal{A}} \pi_\theta(a|s)T(s'|s,a) da$ is irreducible and aperiodic then for $\gamma\rightarrow1$ it holds that
    \begin{align}
        J_{d_b}(\pi) = J_\mu(\pi) \,,
    \end{align}
    for $J_{d_b}$ and $J_\mu$ the normalized objectives.
\end{theorem*}
\begin{proof}
    We start by writing down the difference between the two objectives using their definitions.
    \begin{align}
        J_{d_b}(\pi) - J_\mu(\pi) &= (1-\gamma)\sum_s d_b(s) V_\pi(s) - \sum_s  \mu(s) V_\pi(s)\\
        &= (1-\gamma)\sum_s \left(d_b(s) - \mu(s)\right) V_\pi(s)\\
        &= (1-\gamma)\sum_s \left[\left(d_b(s) - \mu(s)\right)  \sum_{s'}d_{\pi,u}^{s}(s') \int_{\mathcal{A}} r(s',a') \pi(a'|s')da'\right]\\
        &= \sum_s \left[\left(d_b(s) - \mu(s)\right)  \sum_{s'}\mathbf{e}_{s'}^T(1-\gamma)\sum_{t=0}^{\infty}\gamma^t(P_\pi)^t \mathbf{e}_s \int_{\mathcal{A}} r(s',a') \pi(a'|s')da'\right] \,,
    \end{align}
    where $\mathbf{e}_s$ is the column vector with all zeros except in the element corresponding to state $s$, and $d_{\pi,u}^s(s')=\sum_{t=0}^{\infty}\gamma^t Pr(s_t=s'|s_0=s)$ is the un-normalized state visitation for state $s'$ when starting at state $s$.\\\\
    After taking the limit for $\gamma\rightarrow1$ on both sides we apply  Corollary~\ref{lemma:s0_indep} to change the initial state distribution to an arbitrary distribution $\nu\in\Delta(\mathcal{S})$.
    \begin{align}
        \lim_{\gamma\rightarrow1}\left[J_{d_b}(\pi) - J_\mu(\pi) \right]&= \lim_{\gamma\rightarrow1}\sum_s \left[\left(d_b(s) - \mu(s)\right)  \sum_{s'}\mathbf{e}_{s'}^T(1-\gamma)\sum_{t=0}^{\infty}\gamma^t(P_\pi)^t \mathbf{e}_s \int_{\mathcal{A}} r(s',a') \pi(a'|s')da'\right]\\
       (\text{Corollary~\ref{lemma:s0_indep}})\;\;\; &= \sum_s \left[d_b(s) - \mu(s) \right] \sum_{s'}\mathbf{e}_{s'}^TP_\pi^*\, \nu \int_{\mathcal{A}} r(s',a') \pi(a'|s')da'\\      
        &= \left[\sum_s d_b(s) - \sum_s\mu(s) \right] \sum_{s'}\mathbf{e}_{s'}^TP_\pi^*\, \nu \int_{\mathcal{A}} r(s',a') \pi(a'|s') da'\\
        &= 0
    \end{align}
    Through the uniqueness of the stationary distribution, Corollary~\ref{lemma:s0_indep} allows to remove the dependence on the starting state $s$, hence the value function is naturally independent of the initial state distribution making the two objectives equal.\\
\end{proof}
\subsection{Proof Theorem~\ref{th:grad_ub}}
\begin{theorem*}
    (Upper bound for gradient $L_p$-distance)\\
    For an MDP with finite states, compact action space $\mathcal{A}\subseteq [a_l, a_u]^k$, and bounded p-norm $\left\lVert\nabla_\theta\pi_\theta(a|s) \right\rVert_p\; \forall (s,a)\in\mathcal{S}\times\mathcal{A}$, we have
    \begin{align}
        \left\lVert \nabla_\theta J_{d_b} (\pi_\theta) - \nabla_\theta J_\mu (\pi_\theta)\right\rVert_p \leq 2C(a_u - a_l)^k|\mathcal{S}|^{3/2} d_{TV}(d_b||\mu) \,,
    \end{align}
    where $C = \max_{(s,a)}\left\lVert\nabla_\theta\pi_\theta(a|s) \right\rVert_p$.
\end{theorem*}
\begin{proof}
In our derivations we set the interest function to a constant and for practical reasons we set it to $i(s)=1-\gamma\; \forall s \in \mathcal{S}$.
\begin{align}
    \left\lVert \nabla_\theta J_{d_b} (\pi_\theta) - \nabla_\theta J_\mu (\pi_\theta)\right\rVert_p &= \left\lVert (1-\gamma)\sum_s (m(s) - d_\pi(s))\int_{\mathcal{A}} Q_\pi(s,a)\nabla_\theta\pi_\theta(a|s) da \right\rVert_p\\
    (\text{Minkowski's Inequality}) \;\;\; &\leq (1-\gamma)\sum_s \lVert(m(s) - d_\pi(s))\int_{\mathcal{A}} Q_\pi(s,a)\nabla_\theta\pi_\theta(a|s) da\rVert_p\\
    &= (1-\gamma) \sum_s \left|(m(s) - d_\pi(s))\right| \left\lVert\int_{\mathcal{A}} Q_\pi(s,a)\nabla_\theta\pi_\theta(a|s) da\right\rVert_p\\
    (\text{Minkowski's Inequality}) \;\;\; &\leq (1-\gamma)\sum_s \left|(m(s) - d_\pi(s))\right|\int_{\mathcal{A}} \left|Q_\pi(s,a)\right|\left\lVert\nabla_\theta\pi_\theta(a|s) da \right\rVert_p\\
    &\leq (1-\gamma)\max_{(s,a)} \left[\left|Q_\pi(s,a)\right|\left\lVert\nabla_\theta\pi_\theta(a|s) \right\rVert_p\right](a_u - a_l)^k\sum_s \left|(m(s) - d_\pi(s))\right|\\
    &\leq (1-\gamma)\max_{(s,a)} \left|Q_\pi(s,a)\right|\max_{(s,a)}\left\lVert\nabla_\theta\pi_\theta(a|s) \right\rVert_p (a_u - a_l)^k\sum_s \left|(m(s) - d_\pi(s))\right|\\
    \label{eq:31}
    &\leq C(a_u - a_l)^k\sum_s \left|(m(s) - d_\pi(s))\right|
\end{align}
In the last step, we re-named $C :=\max_{(s,a)}\left\lVert\nabla_\theta\pi_\theta(a|s) \right\rVert_p$ and use the fact that the Q-function is upper bounded by $(1-\gamma)^{-1}$.
We can rewrite Equation~\ref{eq:31} in terms of state distribution vectors $m$ and $d_\pi$, 
\begin{align}
    \sum_s \left|(m(s) - d_\pi(s))\right| = \sum_s \left| \mathbf{e}_s^T (m - d_\pi)\right| \, ,
\end{align}
where $\mathbf{e}_s$ is the zero column vector with only the $s$-th element equal to 1. The vector $d_\pi$ can be written as
\begin{align}
     d_\pi &= (1-\gamma) \sum_{t=0}^{\infty} \left[\gamma^t (P_\pi)^t \right]\mu \\
     &= (1-\gamma) (\mathbb{I} - \gamma P_\pi)^{-1} \mu
\end{align}
where $d_\pi$, and $\mu$ are column vectors. Substituting the definition of $m = (\mathbb{I}-\gamma P_\pi)^{-1}d_b$ and the previous expression for $d_\pi$, we get:
\begin{align}
    \left\lVert \nabla_\theta J_{d_b} (\pi_\theta) - \nabla_\theta J_\mu (\pi_\theta)\right\rVert_p &\leq (1-\gamma)C(a_u - a_l)^k\sum_s\left| \mathbf{e}_s^T (\mathbb{I}- \gamma P_{\pi})^{-1}(d_b - \mu)\right|\\
    (\text{Cauchy-Schwarz Ineq.}) \;\;\; &\leq (1-\gamma)C(a_u - a_l)^k\sum_s\left\lVert \mathbf{e}_s^T (\mathbb{I}- \gamma P_{\pi})^{-1}\right\rVert_2 \left\lVert d_b - \mu\right\rVert_2
\end{align}
At this point, we can upper bound the term $\sum_s\left\lVert \mathbf{e}_s^T (\mathbb{I}- \gamma P_{\pi})^{-1}\right\rVert_2$ thanks to the following derivations
\begin{align}
    \sum_s\left\lVert \mathbf{e}_s^T (\mathbb{I}- \gamma P_{\pi})^{-1}\right\rVert_2 &= \sum_s\left\lVert \mathbf{e}_s^T \sum_{t=0}^{\infty}\gamma^t(P_\pi)^t\right\rVert_2\\
    &= \sum_s\left\lVert \sum_{t=0}^{\infty}\gamma^t\mathbf{e}_s^T(P_\pi)^t\right\rVert_2\\
    &\leq \sum_s\left\lVert\sum_{t=0}^{\infty}\gamma^t\mathbf{1}^T\right\rVert_2\\
    &= \sum_s\sum_{t=0}^{\infty}\gamma^t\left\lVert \mathbf{1}^T\right\rVert_2\\
    &= (1-\gamma)^{-1}|\mathcal{S}|^{1/2}\sum_s 1\\
    &= (1-\gamma)^{-1}|\mathcal{S}|^{3/2}\,.
\end{align}
This upper bound on $\sum_s\left\lVert \mathbf{e}_s^T (\mathbb{I}- \gamma P_{\pi})^{-1}\right\rVert_2$ is not particularly tight but still does the job for the purpose of this work. Plugging the final result in the previous expression we obtain:
\begin{align}
    \left\lVert \nabla_\theta J_{d_b} (\pi_\theta) - \nabla_\theta J_\mu (\pi_\theta)\right\rVert_p&\leq (1-\gamma)C(a_u - a_l)^k(1-\gamma)^{-1}|\mathcal{S}|^{3/2} \left\lVert d_b - \mu\right\rVert_2\\
    &= C(a_u - a_l)^k|\mathcal{S}|^{3/2} \left\lVert d_b - \mu\right\rVert_2\\
    (\lVert x\rVert_2 \leq \lVert x \rVert_1) \;\;\; &\leq C(a_u - a_l)^k|\mathcal{S}|^{3/2} \left\lVert d_b - \mu\right\rVert_1\\
    &= 2C(a_u - a_l)^k|\mathcal{S}|^{3/2} d_{TV}(d_b||\mu)
\end{align}
The last expression proves the theorem.
\end{proof}
\subsection{Proof Theorem~\ref{th:grad_ub_w_ergodicty}}
\begin{theorem*}
    (Upper bound for gradient $L_p$-distance with Unique Stationary Distribution)\\
    For an MDP with finite states, compact action space $\mathcal{A}\subseteq [a_l, a_u]^k$, and bounded p-norm $\left\lVert\nabla_\theta\pi_\theta(a|s) \right\rVert_p\; \forall (s,a)\in\mathcal{S}\times\mathcal{A}$, if the Markov chain $P_\pi\in \mathbb{R}^{|\mathcal{S}|\times|\mathcal{S}|}$ with entries $P_{\pi}(s',s)=\int_{\mathcal{A}} \pi_\theta(a|s)T(s'|s,a) da$ is irreducible and aperiodic then we have
    \begin{align}
        \left\lVert \nabla_\theta J_{d_b} (\pi_\theta) - \nabla_\theta J_\mu (\pi_\theta)\right\rVert_p \leq (1-\gamma)2T_mC(a_u - a_l)^k|\mathcal{S}|^{3/2} d_{TV}(d_b||\mu) \,,
    \end{align}
    where $C = \max_{(s,a)}\left\lVert\nabla_\theta\pi_\theta(a|s) \right\rVert_p$.
\end{theorem*}
\begin{proof}
The first part of the proof is the same as in Theorem~\ref{th:grad_ub}, that, for the sake of completeness, we repeat it here.
\begin{align}
    \left\lVert \nabla_\theta J_{d_b} (\pi_\theta) - \nabla_\theta J_\mu (\pi_\theta)\right\rVert_p &= \left\lVert (1-\gamma)\sum_s (m(s) - d_\pi(s))\int_a Q_\pi(s,a)\nabla_\theta\pi_\theta(a|s) \right\rVert_p da\\
    (\text{Minkowski's Inequality}) \;\;\; &\leq (1-\gamma)\sum_s \lVert(m(s) - d_\pi(s))\int_a Q_\pi(s,a)\nabla_\theta\pi_\theta(a|s) \rVert_p \, da\\
    &= (1-\gamma) \sum_s \left|(m(s) - d_\pi(s))\right| \left\lVert\int_a Q_\pi(s,a)\nabla_\theta\pi_\theta(a|s) \right\rVert_p da\\
    (\text{Minkowski's Inequality}) \;\;\; &\leq (1-\gamma)\sum_s \left|(m(s) - d_\pi(s))\right|\int_a \left|Q_\pi(s,a)\right|\left\lVert\nabla_\theta\pi_\theta(a|s) \right\rVert_p da\\
    &\leq (1-\gamma)\max_{(s,a)} \left[\left|Q_\pi(s,a)\right|\left\lVert\nabla_\theta\pi_\theta(a|s) \right\rVert_p\right](a_u - a_l)^k\sum_s \left|(m(s) - d_\pi(s))\right|\\
    &\leq (1-\gamma)\max_{(s,a)} \left|Q_\pi(s,a)\right|\max_{(s,a)}\left\lVert\nabla_\theta\pi_\theta(a|s) \right\rVert_p(a_u - a_l)^k\sum_s \left|(m(s) - d_\pi(s))\right|\\
    \label{eq:54}
    &\leq C(a_u - a_l)^k\sum_s \left|(m(s) - d_\pi(s))\right|
\end{align}
In the last step, we re-named $C :=\max_{(s,a)}\left\lVert\nabla_\theta\pi_\theta(a|s) \right\rVert_p$ and use the fact that the Q-function is upper bounded by $(1-\gamma)^{-1}$. We can rewrite Equation~\ref{eq:54} in terms of state distribution vectors $m$ and $d_\pi$, 
\begin{align}
    \sum_s \left|(m(s) - d_\pi(s))\right| = \sum_s \left| \mathbf{e}_s^T (m - d_\pi)\right| \, ,
\end{align}
where $\mathbf{e}_s$ is the zero column vector with only the $s$-th element equal to 1. The vector $d_\pi$ can be written as
\begin{align}
     d_\pi &= (1-\gamma) \sum_{t=0}^{\infty} \left[\gamma^t (P_\pi)^t \right]\mu \\
     &= (1-\gamma) (\mathbb{I} - \gamma P_\pi)^{-1} \mu
\end{align}
where $d_\pi$, and $\mu$ are column vectors. Substituting the definition of $m = (\mathbb{I}-\gamma P_\pi)^{-1}d_b$ and the previous expression for $d_\pi$, we get:
\begin{align}
    C(a_u - a_l)^k\sum_s \left|(m(s) - d_\pi(s))\right| &= (1-\gamma)C(a_u - a_l)^k\sum_s\left| \mathbf{e}_s^T (\mathbb{I}- \gamma P_{\pi})^{-1}(d_b - \mu)\right|
\end{align}
To separate the transitory and stationary parts, we express the term $(\mathbb{I}- \gamma P_{\pi})^{-1}$ as a geometric series.
\begin{align}
    C(a_u - a_l)^k\sum_s \left|(m(s) - d_\pi(s))\right| & = (1-\gamma)C(a_u - a_l)^k\sum_s\left| \mathbf{e}_s^T \left[\sum_{t=0}^\infty\gamma^t (P_\pi)^t\right](d_b - \mu)\right|
\end{align}
To use the additional irreducibility and aperiodicity assumptions we split the sum before and after the strong stationary time $T_m$.
\begin{align}
    \left\lVert \nabla_\theta J_{d_b} (\pi_\theta) - \nabla_\theta J_\mu (\pi_\theta)\right\rVert_p & = C(a_u - a_l)^k\sum_s\left| \mathbf{e}_s^T \left[(1-\gamma)\sum_{t=0}^{T_m-1}\gamma^t (P_\pi)^t + (1-\gamma) \sum_{t=T_m}^{\infty}\gamma^t (P_\pi)^t\right](d_b - \mu)\right|\\
    & = C(a_u - a_l)^k\sum_s\left| \mathbf{e}_s^T \left[(1-\gamma)\sum_{t=0}^{T_m-1}\gamma^t (P_\pi)^t + (1-\gamma) \sum_{t=T_m}^{\infty}\gamma^t P_\pi^*\right](d_b - \mu)\right|\\
    & = C(a_u - a_l)^k\sum_s\left| \mathbf{e}_s^T \left[(1-\gamma)\sum_{t=0}^{T_m-1}\gamma^t (P_\pi)^t + (1-\gamma) \frac{\gamma^{T_m}}{1-\gamma} P_\pi^*\right](d_b - \mu)\right|\\
    & = C(a_u - a_l)^k\sum_s\left| \mathbf{e}_s^T \left[(1-\gamma)\sum_{t=0}^{T_m-1}\gamma^t (P_\pi)^t + \gamma^{T_m} P_\pi^*\right](d_b - \mu)\right|
\end{align}
Because of the uniqueness of the stationary distribution, we know that $P_\pi^*$ is independent on the starting state and therefore has identical columns. This is equivalent to stating that $P_\pi^* d_b=P_\pi^* \mu$ for any $d_b,\mu \in \Delta(\mathcal{S})$, which renders the element $\gamma^{T_m} P_\pi^*(d_b - \mu)$ equal to zero.
\begin{align}
    \left\lVert \nabla_\theta J_{d_b} (\pi_\theta) - \nabla_\theta J_\mu (\pi_\theta)\right\rVert_p & \leq C(a_u - a_l)^k\sum_s\left| \mathbf{e}_s^T \left[(1-\gamma)\sum_{t=0}^{T_m-1}\gamma^t(P_\pi)^t + \gamma^{T_m} P_\pi^*\right](d_b - \mu)\right|\\
    (\gamma^t(P_\pi)^t\leq(P_\pi)^t)\;\;\;&\leq C(a_u - a_l)^k\sum_s\left| \mathbf{e}_s^T \left[(1-\gamma)\sum_{t=0}^{T_m-1}(P_\pi)^t + \gamma^{T_m} P_\pi^*\right](d_b - \mu)\right|\\
    & =(1-\gamma) C(a_u - a_l)^k\sum_s\left| \mathbf{e}_s^T \sum_{t=0}^{T_m-1}(P_\pi)^t (d_b - \mu)\right|\\
    (\text{Cauchy-Schwarz Ineq.}) \;\;\; &\leq(1-\gamma) C(a_u - a_l)^k\sum_s\left\lVert \mathbf{e}_s^T \sum_{t=0}^{T_m-1}(P_\pi)^t\right\rVert_2 \left\lVert d_b - \mu\right\rVert_2\\
     &\leq(1-\gamma) C(a_u - a_l)^k\sum_s\left\lVert \sum_{t=0}^{T_m-1}\mathbf{e}_s^T(P_\pi)^t\right\rVert_2 \left\lVert d_b - \mu\right\rVert_2\\
    &\leq(1-\gamma) C(a_u - a_l)^k\sum_s\left\lVert T_m \mathbf{1}^T \right\rVert_2 \left\lVert d_b - \mu\right\rVert_2\\
    &=(1-\gamma) T_m C(a_u - a_l)^k|\mathcal{S}|^{3/2} \left\lVert d_b - \mu\right\rVert_2\\
    (\lVert x\rVert_2 \leq \lVert x \rVert_1) \;\;\; &\leq (1-\gamma) T_m C(a_u - a_l)^k|\mathcal{S}|^{3/2} \left\lVert d_b - \mu\right\rVert_1\\
    &= (1-\gamma) 2T_m C(a_u - a_l)^k|\mathcal{S}|^{3/2} d_{TV}(d_b||\mu)
\end{align}
\end{proof}
\subsection{Proof Corollary~\ref{cor:grad_equiv}}
\begin{corollary*}
    For an MDP with finite state, , compact action space $\mathcal{A}\subseteq [a_l, a_u]^k$, and bounded p-norm $\left\lVert\nabla_\theta\pi_\theta(a|s) \right\rVert_p\; \forall (s,a)\in\mathcal{S}\times\mathcal{A}$, if the Markov chain $P_\pi\in \mathbb{R}^{|\mathcal{S}|\times|\mathcal{S}|}$ with entries $P_{\pi}(s',s)=\int_{\mathcal{A}} \pi_\theta(a|s)T(s'|s,a) da$ is irreducible and aperiodic, then for $\gamma\rightarrow1$ it holds that
    \begin{align}
        \nabla_\theta J_{d_b} (\pi_\theta) = \nabla_\theta J_\mu(\pi_\theta)
    \end{align}
\end{corollary*}
\begin{proof}
    The proof straightforwardly follows from taking the limit of the upper bound presented in Theorem~\ref{th:grad_ub_w_ergodicty} as
    \begin{align}
        \lim_{\gamma\rightarrow1}\left\lVert \nabla_\theta J_{d_b} (\pi_\theta) - \nabla_\theta J_\mu (\pi_\theta)\right\rVert_p & \leq \lim_{\gamma\rightarrow1} (1-\gamma) 2T_mC(a_u - a_l)^k|\mathcal{S}|^{3/2} d_{TV}(d_b||\mu) \, ,
    \end{align}
    which allows us to state that for $\gamma\rightarrow1$ then 
    \begin{align}
        \lim_{\gamma\rightarrow1} \left\lVert \nabla_\theta J_{d_b} (\pi_\theta) - \nabla_\theta J_\mu (\pi_\theta)\right\rVert_p \leq 0 \, .
    \end{align}
    The consequence is that for $\gamma\rightarrow1$ it holds that
    \begin{align}
        \nabla_\theta J_{d_b} (\pi_\theta) =  \nabla_\theta J_\mu (\pi_\theta) \, .
    \end{align}
\end{proof}

\section{Appendix B - Experimental Setup}
\label{sec:app_exp_setup}
\subsection{Two states MDP}
\begin{figure}[h!]
    \centering
    \begin{tikzpicture}[roundnode/.style={circle, draw=black, very thick, minimum size=10mm}]
            % Nodes
            \node[roundnode](s0) at (0,0) {$s_0$};
            \node[roundnode](s1) at (3,0) {$s_1$};
            % Lines
            \draw[->] (s0) to[bend right] node[midway,below,inner sep=2pt] {``move"} (s1);
            \draw[<-] (s0) to[bend left] node[midway,above,inner sep=2pt] {``move"} (s1);
            \draw[<-] (s0) to[out=-150,in=150,loop] node[midway,left,inner sep=2pt] {``stay"} ();
            \draw[<-] (s1) to[out=30,in=-30,loop] node[midway,above,inner sep=13pt] {``stay"} ();
    \end{tikzpicture}
    \caption{The environment executes the action the agent selects with probability $q$. Therefore, $Pr(s|\text{``stay",\,s})=Pr(\neg s |\text{``move",\,s})=q$. The reward function is $r(s_1,\,-)=1$ and $r(s_0,\,-)=0$.}
    \label{fig:two_state_mdp}
\end{figure}
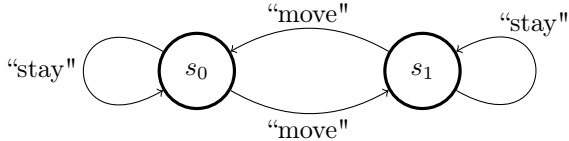
For our empirical experiments, we use the two-state MDP depicted in Figure~\ref{fig:two_state_mdp}, where an agent can either move between states or stay at the current state. The agent is rewarded for staying in state $s_1$ but the environment includes slippage meaning that the environment executes the action the agent selects only with probability $q$ as highlighted in Figure~\ref{fig:two_state_mdp}. Across our experiments the behavioral policy $b$ has been selected to ``stay" with 0.9 probability and the starting state is sampled from a uniform distribution. This environment satisfies all our assumptions having a finite-state space and a strongly connected state-transition graph.

The policies are parametrized by $p$, such that $\pi(``\text{stay}"|s_1)=\pi(``\text{move}"|s_0)=p$, and we sample policies by sampling $p\in[0,1]$ uniformly. Crucially, because of our choice of the policy parametrization, we have that $\nabla_\theta\pi_\theta(a|s)=1$, hence $\left\lVert\nabla_\theta\pi_\theta(a|s) \right\rVert_p$ is bounded for all state-action pairs, matching our assumptions.

Since both objectives rely on estimating at first the per-state value $V_\pi(s)$, we use for all our experiments a tabular representation for the value function estimate and Expected SARSA \citep{sutton2018}, an off-policy variant of the better-known on-policy temporal difference algorithm SARSA. We update the value function estimate one hundred thousand times with 0.5 as the learning rate.

\subsection{DeepMind Control Suite}
\label{subsec:Mujoco}
\begin{figure}[ht]
\centering
\begin{subfigure}{0.32\textwidth}
    \includegraphics[width=\textwidth]{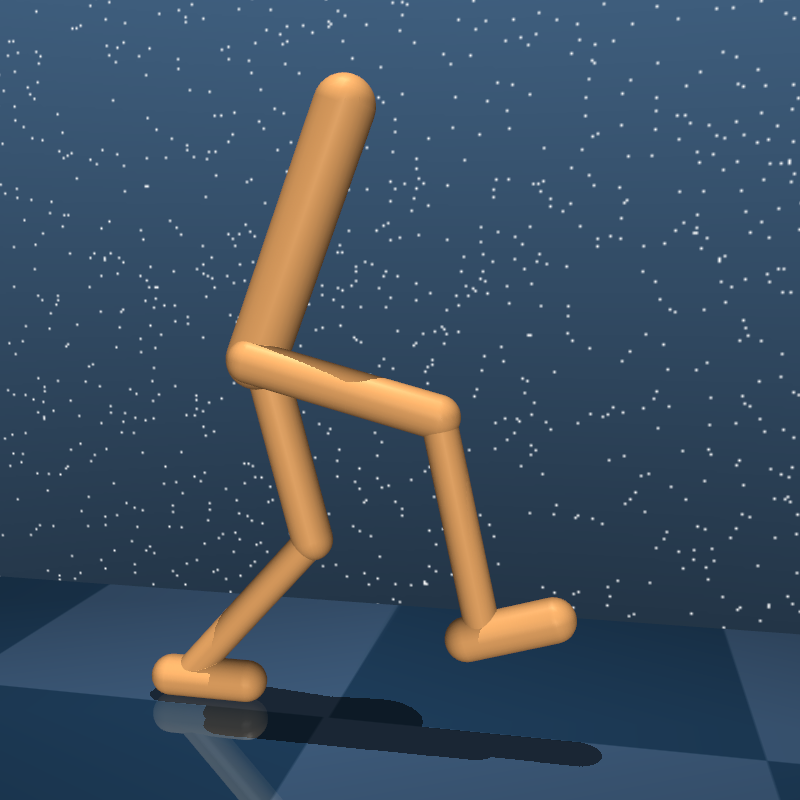}
    \caption{}
\end{subfigure}
\hfill
\begin{subfigure}{0.32\textwidth}
    \includegraphics[width=\textwidth]{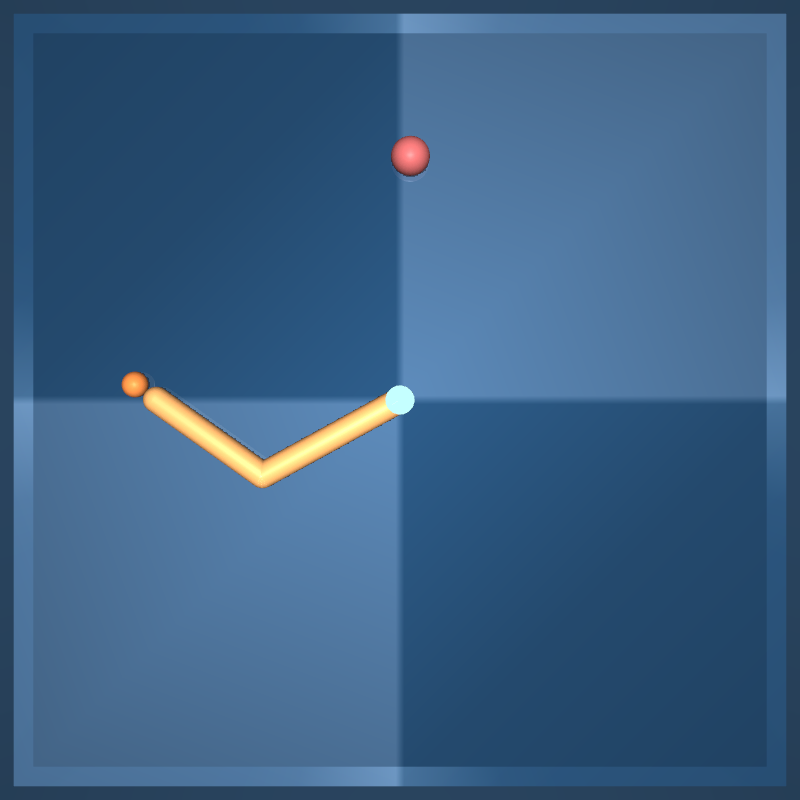}
    \caption{}
\end{subfigure}
\hfill
\begin{subfigure}{0.32\textwidth}
    \includegraphics[width=\textwidth]{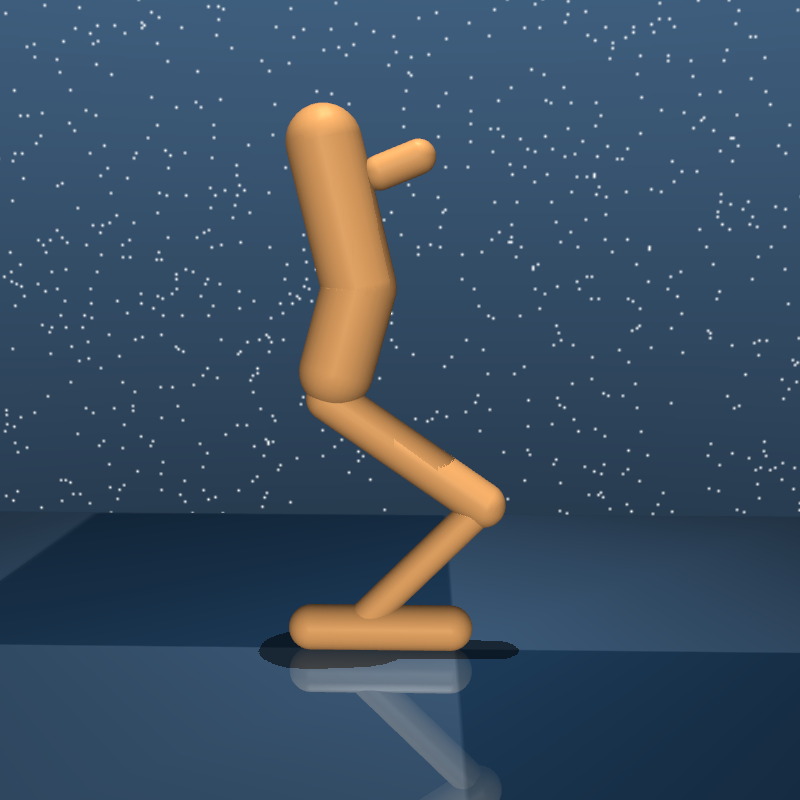}
    \caption{}
\end{subfigure}
\caption{(a) Walker. (b) Reacher. (c) Hopper.}
\label{fig:mujoco_img}
\end{figure}

We use the environments Walker, Hopper, and Reacher (see Figure~\ref{fig:mujoco_img}) for our experiments with the DeepMind control suite \citet{tunyasuvunakool2020}. To have diverse performances among the policies to rank, hence a meaningful ranking, we train for each environment a policy achieving good performances using off-the-shelf RL algorithms. For each environment, we then generate 100 Gaussian policies by sampling uniformly different variance values to and using as mean the action selected by the good-performing policy obtained through training. We then evaluate the true on-policy objective achieved by each policy averaging the returns obtained over 1000 episodes. To compute the excursion objective we run 1000 episodes sampling the starting state from a dataset obtained with a uniformly random policy. This way of computing the excursion objective aligns with our theoretical analysis as we do not consider errors due to estimating off-policy the value function, but only the mismatch in the starting state distribution.

In Figure~\ref{fig:mujoco_exp}, for each value of $\gamma$, we sample 15 policies from the 100 policies we previously generated and compute Kendall's $\tau$ between the ranking generated by the excursion objective and the one generated by the on-policy objective. To obtain the 95\% confidence interval this is repeated 30 times.

Figure~\ref{fig:mujoco_pvalue} is complementary to Figure~\ref{fig:mujoco_exp}, since, instead of showing the empirical 95\% confidence interval of the ranking metric, it displays the $p$-value for the Null hypothesis of the two rankings being independent. Figure~\ref{fig:mujoco_pvalue} allows us to make strong statements about the correlation significance when we get $p$-values below the $0.05$ threshold. On the other hand, Figure~\ref{fig:mujoco_exp} enables us to discuss with high confidence the ranking performances achieved independently from the correlation between the two objectives.

\begin{figure*}[t]
\centering
\includegraphics[width=\textwidth]{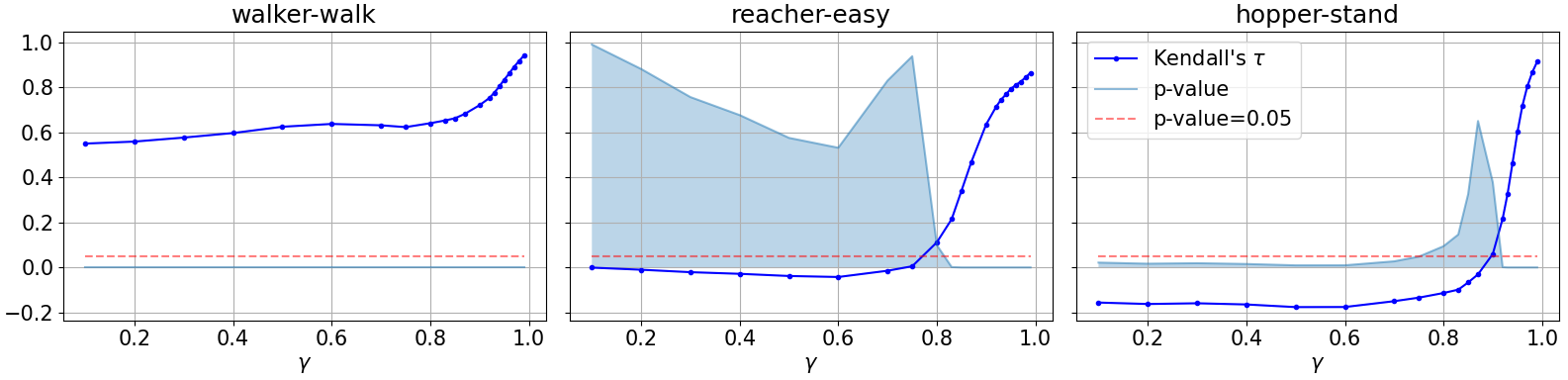}
\caption{Correlation coefficient and $p$-value in offline policy selection across 3 Mujoco environments as a function of $\gamma$ when ranking a set of 100 policies based on the excursion objective, with 1 being the same ranking as ground truth.}
\label{fig:mujoco_pvalue}
\end{figure*}

In Figure~\ref{fig:mujoco_pvalue}, we compute Kendall's $\tau$ and its associated $p$-value when ranking all the previously generated 100 policies. In Walker, the two objectives are significantly and positively correlated across the whole spectrum of $\gamma$, with a correlation coefficient increasing with $\gamma$. In Reacher and Hopper, we can observe that, in accordance with our theoretical results, for $\gamma$ close to 1, there is a significant increase in correlation between the two objectives, providing clear proof that for this regime lower values of $\gamma$ can cause the excursion objective to align less with the on-policy performances. By observing a low significant correlation between the two objectives, Figure~\ref{fig:mujoco_pvalue} shows how misleading can be to optimize the excursion objective even at high $\gamma$ regimes. Importantly, in both Reacher and Hopper, we observe regimes with high $p$-values where there is no evidence of any correlation between the two rankings. For Hopper, for $\gamma<0.7$ there is even a significant anti-correlation between the two rankings.

\textbf{Summary:} Even for continuous action spaces, our experiments show that despite $\gamma\approx0.9$ there can be a low correlation (or no evidence of correlation) between the two objectives. This highlights the danger of using the excursion objective for $\gamma$ not close enough to 1. 

\end{document}